\documentclass[letterpaper]{article}
\usepackage{proceed2e}
\usepackage[margin=1in]{geometry}

\usepackage{times}
\usepackage[linesnumbered,ruled]{algorithm2e}
\let\chapter\undefined
\usepackage{varioref}
\usepackage{graphicx}
\usepackage{subfig}
\usepackage{amsmath}
\usepackage{amssymb}
\usepackage{amsthm}
\usepackage{stmaryrd}
\usepackage{paralist}
\usepackage{xcolor}
\usepackage{soul}
\usepackage[aboveskip = 6pt]{caption}
\usepackage{comment}
\usepackage{tabularx}
\usepackage{multirow}
\usepackage{booktabs}
\usepackage[figuresright]{rotating}
\usepackage[colorlinks, linkcolor=blue, citecolor=blue]{hyperref}
\usepackage{appendix}
\usepackage{CJK}
\theoremstyle{definition}
\newtheorem{definition}{Definition}
\theoremstyle{plain}
\newtheorem{theorem}{Theorem}
\newtheorem{proposition}{Proposition}

\newtheorem{lemma}{Lemma}
\theoremstyle{remark}

\newtheorem{example}{Example}
\newcommand{\textquoted}[1]{\textquotedblleft #1\textquotedblright{}}
\graphicspath{{../Figures/}{../../Figures/}{../Figures/experiments/}{../Figures/variable-order/}{../Figures/OBDDCs/}}

\newcommand{\Angle}[1]{\langle #1 \rangle}

\newcommand{\Expectation}[2]{E_{#1}\left[#2\right]}

\newcommand{\NOT}{\neg}
\newcommand{\AND}{\wedge}
\newcommand{\OR}{\vee}

\newcommand{\DEC}{\diamond}

\newcommand{\OBDDC}[2]{OBDD$[\AND_{#1}]_{#2}$}

\newcommand{\OBDDL}[1]{OBDD-$L_{#1}$}

\DontPrintSemicolon
\SetAlgoVlined
\SetKwBlock{KwFunc}{function}{end}

\newcommand{\SampleDNNF}{\textsc{SampleDNNF}}
\newcommand{\RandomPart}{\textsc{RandomPart}}
\newcommand{\MicroKC}{\textsc{MicroKC}}
\newcommand{\PartialKC}{\textsc{PartialKC}}
\newcommand{\PartialOBDDL}{\textsc{PartialOBDDL}}

\title{Approximate Model Counting by Partial Knowledge Compilation}

\author{ {\bf Yong Lai \thanks{Key Laboratory of Symbolic Computation and Knowledge Engineering of Ministry of Education, Changchun, China.}} \\
College of Computer Science and Technology \\
Jilin University \\
Changchun 130012, China \\
}

\begin{document}

\maketitle

\begin{abstract}
Model counting is the problem of computing the number of satisfying assignments of a given propositional formula.
Although exact model counters can be naturally furnished by most of the knowledge compilation (KC) methods, in practice, they fail to generate the compiled results for the exact counting of models for certain formulas due to the explosion in sizes.
Decision-DNNF is an important KC language that captures most of the practical compilers.
We propose a generalized Decision-DNNF (referred to as partial Decision-DNNF) via introducing a class of new leaf vertices (called unknown vertices), and then propose an algorithm called \PartialKC{} to generate randomly partial Decision-DNNF formulas from the given formulas.
An unbiased estimate of the model number can be computed via a randomly partial Decision-DNNF formula.
Each calling of \PartialKC{} consists of multiple callings of \MicroKC{}, while each of the latter callings is a process of importance sampling equipped with KC technologies. The experimental results show that \PartialKC{} is more accurate than both SampleSearch and SearchTreeSampler, \PartialKC{} scales better than SearchTreeSampler, and the KC technologies can obviously accelerate sampling.
\end{abstract}

\section{Introduction}
\label{sec:intro}

Knowledge compilation (KC) is concerned with converting general types of knowledge into tractable forms, which allows certain hard reasoning tasks to be performed efficiently on the compiled forms \cite{Selman:Kautz:96,Darwiche:Marquis:02,Cadoli:Donini:97}.
Model counting is the problem of computing the number of satisfying assignments of a given propositional formula, which is instrumental to the effective performance of probabilistic inference \cite{Roth:96,Bacchus::etal:03,Sang:etal:05,Chavira:Darwiche:08}.
It is well known that the exact model count computing is a \#P-complete problem.
According to the idea of KC, for a propositional language supporting tractable model counting, one can first compile a formula into the target language, and then count the models of the compiling result in polytime (see e.g., \cite{Koriche:etal:13}).

A common deficiency of all KC approaches is the explosion in sizes for certain types of formulas in practice, and thus we cannot generate the compiled results to count the exact model numbers of such types of formulas. In particular, a compiler is often relatively inefficient in the conjunctive normal form (CNF) formulas with high treewidths. This brings the issue of seeking an alternative solution via a partial compilation of hard CNF formulas, with a further estimation of a model number based on the partial compiled results obtained.

Decomposable Negation Normal Form (DNNF) \cite{Darwiche:01a} is an influential KC language which includes Ordered Binary Decision Diagrams (OBDDs) \cite{Bryant:86}, Sentential Decision Diagrams (SDDs) \cite{Darwiche:11}, OBDDs with the conjunctive decomposition (\OBDDC{}{}) \cite{JAIR17}. Decision-DNNF \cite{Oztok:Darwiche:14} is a subset of DNNF that captures most of the practical knowledge compilers, including c2d \cite{Darwiche:04}, \textsc{Dsharp} \cite{Dsharp}, BDDjLu \cite{KAIS12}, and D4 \cite{D4}.
A generalization of Decision-DNNF (hereinafter referred to as partial Decision-DNNF) is proposed in this study by introducing a class of new leaf vertices (the so-called unknown vertices).
Each unknown vertex represents a sub-formula (conditioned on some partial assignment) that is not compiled yet.
We propose an algorithm called \PartialKC{} to generate a randomly partial Decision-DNNF.
We can compute an unbiased estimate of model number using the randomly partial Decision-DNNF.
Each calling of \PartialKC{} consists of multiple callings of \MicroKC{}.
Each of the latter callings is treated as a process of importance sampling equipped with KC technologies, including dynamic decomposition, non-chronological backtracking, component caching, and backbone-based simplification \cite{Huang:Darwiche:07,KAIS12}.
The proposal distribution chosen in \MicroKC{} is based on the partial Decision-DNNF formula itself.
To ensure that each partial Decision-DNNF formula can represent a joint distribution, its decision vertices are labeled with certain probabilities.
The probabilities are estimated by a DPLL-based algorithm which generates special partial Decision-DNNF formulas.

Many approximate model counters are based on the Monte Carlo methods (see e.g., \cite{Gomes:etal:07,Gogate:Dechter:11,Ermon:etal:13,ApproxMC}).
Such counters can be classified into three basic categories \cite{ApproxMC}.
Let $\varphi$ be a formula with $Z$ models.
A counter of the first category is parameterized by $(\varepsilon, \delta)$, and computes a model number of $\varphi$ that lies in the interval $[(1 + \varepsilon)^{-1}Z, (1 + \varepsilon)Z]$ with confidence at least $1 - \delta$.
A counter of the second category is parameterized by $\delta$, and computes a lower (or upper) bound of $Z$ with confidence at least $1 - \delta$. The counters of the third category provide fewer guarantees, but offer proper approximations in practice.
These counters often scales better than the ones in the first category.
The Markov's inequality application allows one to convert the majority of third-category counters into those of the second category \cite{Gomes:etal:07}.
SampleSearch \cite{Gogate:Dechter:11} and SearchTreeSampler \cite{Ermon:etal:13} are the state-of-the-art counters of the third category.
The former scales better while the latter often provides more accurate estimates.
\PartialKC{} also falls into the third category.
The experimental results show that \PartialKC{} is more accurate than both SampleSearch and SearchTreeSampler, and scales better than SearchTreeSampler.

\paragraph{Related Work.}
This work is closely related to the study \cite{Gogate:Dechter:12}.
For many KC languages, each model of a formula can be seen a particle of the corresponding compiled result.
Therefore, the authors \cite{Gogate:Dechter:12} used the generated models (i.e., samples) to construct an AND/OR sample graph, and then used it to estimate the model number of a CNF formula.
Each AND/OR sample graph can be treated as a partial compiled result in AOBDD (binary version of AND/OR Multi-Valued Decision Diagram \cite{Mateescu:etal:08}).
It was theoretically proved in \cite{Gogate:Dechter:12} that the estimate variance of the partial AOBDD is smaller than that of the mean of samples.
The proposed \PartialKC{} approach has two main differences from that of \cite{Gogate:Dechter:12}.
Firstly, the latter approach envisages an independent generation of each sample, while the KC technologies used in \PartialKC{} can accelerate the sampling (and thus the convergence), which fact is experimentally verified in this study.
Secondly, the decomposition used by the partial AOBDD in \cite{Gogate:Dechter:12} is static, while that used by partial Decision-DNNF is dynamic, which  is known to be more effective in the KC field than the static one \cite{Dsharp,D4}.

\section{Preliminaries}
\label{sec:preli}

This paper uses $x$ to denote a propositional or Boolean variable, and $X$ to denote a set of variables.
To simplify notations, a singleton set is assimilated, in some cases, with its unique element.
A formula is constructed from constants $0$, $1$ and variables using negation operator $\NOT$, conjunction operator $\AND$, and other logical operators that can be defined using $\NOT$ and $\AND$.
For example, the disjunction operator $\OR$ and decision operator $\DEC$ can be defined as $\varphi \vee \psi = \NOT( \NOT\varphi \wedge \NOT\psi)$ and $\varphi \DEC_x \psi = (\NOT x \AND \varphi) \OR (x \AND \psi)$, respectively.
As applied  to formula $\varphi$, we use $Vars(\varphi)$ to denote the set of variables appearing in $\varphi$.

An assignment $\omega$ over variable set $X$ is a mapping from $X$ to $\{0, 1\}$, and the set of all assignments over $X$ is denoted by $2^X$. Given any formula $\varphi$ and assignment $\omega$ over a superset of $Vars(\varphi)$, $\omega$ satisfies $\varphi$ (denoted by $\omega \models \varphi$) iff one of the following conditions holds: $\varphi = 1$; $\varphi = x$ and $\omega(x) = 1$; $\varphi = \NOT \psi$ and $\omega \not\models \psi$; or $\varphi = \psi \AND \psi'$, $\omega \models \psi$ and $\omega \models \psi'$.
A formula is satisfiable if it has at least one model, and it is unsatisfiable otherwise.
The number of models of formula $\varphi$ over $X$ is denoted by $Z_X(\varphi)$, and $X$ is sometimes omitted in an explicit context.
The marginal probability of $\varphi$ over a variable $x \in X$ is the ratio of $Z_X(\varphi \AND x)$ to $Z_X(\varphi)$.
Model counting, also known as \#SAT, concerns counting the model number of a given formula.
Given two formulas $\varphi$ and $\psi$, $\varphi$ is equivalent to $\psi$ (denoted by $\varphi \equiv \psi$) iff the model set of $\varphi$ and that of $\psi$ over $Vars(\varphi) \cup Vars(\psi)$ are equal to each other.
The conditioning of formula $\varphi$ on assignment $\omega$, denoted by $ \varphi |_\omega$, is a formula obtained by replacing each $x$ in $\varphi$ with 1 (resp. 0) if $x=1 \in \omega$ (resp. $x = 0 \in \omega$).

A literal is either a variable $x$ or its negation $\NOT x$. Given a literal $l$, its negation $\NOT l$ is $\NOT x$ if $l$ is $x$, and $\NOT l$ is $x$ otherwise.
A literal $l$ is called an \emph{implied literal} of formula $\varphi$ if $\varphi \models l$; and an implied literal can be used to simplify $\varphi$ by conditioning.
A clause $\delta$ is a set of literals representing their disjunction. A formula in conjunctive normal form (CNF) is a set of clauses representing their conjunction.

A formula in \emph{negation normal form} (NNF) is constructed from 0, 1 and literals using only conjoining and disjoining operators. An NNF formula can be represented as a rooted, directed acyclic graph (DAG) where each leaf vertex is labeled with $0$, $1$ or literal; and each internal vertex is labeled with $\AND$ or $\OR$ and can have arbitrarily many children. For an NNF vertex $v$, we sometimes use $\vartheta_v$ to denote the NNF formula rooted at $v$. An NNF formula $\vartheta$ is \emph{decomposable} (DNNF) if for each $\AND$-vertex $u$ in $\vartheta$, the sub-formulas rooted at the children $Ch(u)$ do not share variables.

\section{Full Decision-DNNF and Partial Decision-DNNF}
\label{sec:DNNF}

Decision-DNNF \cite{Oztok:Darwiche:14} is a subset of DNNF that captures most of the practical knowledge compilers, and supports tractable model counting.
Therefore, each Decision-DNNF compiler directly provide an exact model counter.
A well known problem of Decision-DNNF is the explosion in sizes for certain types of formulas in practice.
In other words, we cannot use the compiled results to count the exact model numbers of such hard formulas.
In particular, the size of a Decision-DNNF formula often explodes when representing a CNF formula with high treewidth.
We propose a generalization called \emph{partial Decision-DNNF} to overcome this problem.
For notational convenience, a Decision-DNNF formula is also referred to as a \emph{full} Decision-DNNF formula.
In this section, we present the notions and properties of full and partial Decision-DNNF.

\subsection{Full Decision-DNNF}
\label{sec:DNNF:full}

We first present the notion of full Decision-DNNF, and we can use a full Decision-DNNF formula to represent a joint probability distribution by labeling some arcs with probabilities:
\begin{definition}[full Decision-DNNF]
A full Decision-DNNF formula is a rooted DAG.
Each vertex $v$ is labeled with a symbol $sym(v)$. If $v$ is a leaf, $sym(v) = \bot$ or $\top$. Otherwise, $sym(v)$ is a variable (in which case $v$ is called a \emph{decision vertex}) or operator $\AND$ (called a \emph{decomposition vertex}).
Each internal vertex $v$ has a set of children $Ch(v)$. For a decision vertex, $Ch(v) = \{ch_0(u), ch_1(u)\}$, where $lo(v)$ and $hi(v)$ are called \emph{low} and \emph{high} children, and are connected by dashed and solid arcs, respectively; for a decomposition vertex, the subgraphs rooted at children share no variables.
Each arc from a decision vertex $v$ is labeled with an estimated marginal probability $p_b(v)$ , where $b = 0 \text{ or } 1$ means the arc is dashed or solid.
For a decision vertex $v$, $p_0(v) + p_1(v) = 1$, and $p_b(v) = 0$ iff $sym(ch_b(v)) = \bot$.
\end{definition}

Figure \ref{fig:DecDNNF:a} depicts a full Decision-DNNF. Hereafter we denote a leaf by $\Angle{\bot}$ or $\Angle{\top}$, a decomposition vertex by $\Angle{\AND, Ch(v)}$, and a decision vertex by $\Angle{sym(v), ch_0(v), ch_1(v)}$.
For simplicity, we sometimes use $\Angle{x}$ to denote $\Angle{x, \Angle{\bot}, \Angle{\top}}$; $\Angle{\NOT x}$ to denote $\Angle{x, \Angle{\top}, \Angle{\bot}}$; and a full Decision-DNNF formula rooted at $u$ by $\vartheta_u$.
Each leaf $\Angle{\bot}$ or $\Angle{\top}$ represent 0 or 1.
Each decision vertex $v$ labeled with $x$ represents formula $\varphi \DEC_x \psi$ where $\varphi$ and $\psi$ represent the formulas rooted at $ch_0(v)$ and $ch_1(v)$, respectively, and $x$ appears in neither $\varphi$ nor $\psi$.
Each decomposition vertex represents a conjunction of the formulas rooted at its children.
It is easy to see that Decision-DNNF is a subset of DNNF, and each decision vertex does not have two constants 0 as its children.

Given a Decision-DNNF vertex $u$ over $X$ and a vertex $v$ in $\vartheta_u$, we use $Z(v)$ to denote the number of models of $\vartheta_v$ over $X$.
Then $Z(u)$ can be recursively counted in linear time ($c = 2^{(1- |Ch(u)|) \cdot |X|}$):
\begin{equation}\label{eq:sharp-DecDNNF}
Z(u) = \begin{cases}
0 \text{ or } 2^{|X|} & \text{$u$ is a leaf}; \\
c \cdot \prod_{v \in Ch(u)}Z(v) & \text{$u$ is a $\AND$-vertex};  \\
\dfrac{Z(ch_0(u)) + Z(ch_1(u))}{2} & {\text{otherwise.}}
\end{cases}
\end{equation}

A Decision-DNNF formula $\vartheta$ is an \emph{Ordered Binary Decision Diagram with implied Literals} (\OBDDL{}) \cite{KAIS12} over a linear order $\prec$ of variables if it satisfies two conditions:
for each decomposition vertex in $\vartheta$, there exists at most one child representing a non-literal;
for each decision vertex $v$ and its decision descendant $w$, the variable labeled on $v$ is less than the one labeled on $w$ over $\prec$.

We can use a Decision-DNNF formula to represent a joint probability distribution. Let $u$ be a Decision-DNNF vertex over $X$. For an assignment $\omega$ over $X$, $u$ defines the distribution as follows:
\begin{compactitem}[$\bullet$]
\item If $u$ is 0, then $Pr_u(\omega) = 0$;
\item If $u$ is 1, then $Pr_u(\omega) = 2^{-|X|}$;
\item If $u$ is a decision vertex, then $Pr_u(\omega) = 2 \cdot p_b(u) \cdot Pr_{ch_b(u)}(\omega)$, where $x = b \in \omega$; and
\item If $u$ is a $\AND$-vertex, then $Pr_u(\omega) = 2^{(|Ch(u)| - 1) \cdot |X|} \cdot \prod_{i=1}^{n}Pr_{v_i}(\omega)$.
\end{compactitem}
Therefore, each non-zero vertex represents a joint probability distribution over $X$.
Consider the full Decision-DNNF formula in Figure \ref{fig:DecDNNF:a} again.
We denote the root by $u$, and have that $Pr_u(x_1 = 0, x_2 = 0, x_4 = 1, x_6 = 0) = 0.24$ and $Pr_u(x_1 = 1, x_2 = 1, x_4 = 1, x_4 = 1, x_6 = 0) = 0.12$.

\begin{proposition}\label{prop:cond-pr-DecDNNF}
Let $u$ be a full Decision-DNNF vertex, let $v$ be a decision vertex in $\vartheta_u$, and let $X$ be the set of variables labeled on some path from $u$ to a parent of $v$. We have that $Pr_u(sym(v) = b|X) = p_b(v)$.
\end{proposition}

According to the above proposition, we can perform importance sampling with a full Decision-DNNF formula as the proposal distribution. The sampling algorithm is presented in Algorithm \ref{alg:SampleDNNF}.

\begin{algorithm}[htbp]
\caption{\SampleDNNF($u$)} \label{alg:SampleDNNF}
\KwIn{a satisfiable full Decision-DNNF vertex $u$}
\KwOut{a partial assignment $\omega$ with probability $Pr_u(\omega)$}
\lIf {$sym(u) = \top$} {\KwRet{$\emptyset$}}\;
\uElseIf {$sym(u) \in PV$} {
    Sample a Boolean value $b$ with probability $p_1(u)$\;
    \KwRet{$\{sym(u) = b\} \cup \SampleDNNF(ch_b(u))$}\;
}
\lElseIf {$sym(u) = \AND$} {$\bigcup_{v \in Ch(u)} \SampleDNNF(v)$}\;
\end{algorithm}

Let $u$ be a full Decision-DNNF vertex over $X$.
For a set of outputs $\omega_1, \ldots, \omega_N$ of Algorithm \SampleDNNF($u$), an unbiased estimate of $Z(u)$ can be derived as follows:
\begin{equation*}\label{eq:sharp-DecDNNF-Decision}
\widehat{Z}_N(u) = \frac{1}{N} \sum_{i=1}^{N} \frac{2^{|X|-|\omega_i|}}{Pr_u(\omega_i)}
\end{equation*}
If $u$ is a decision vertex, and $Q$ is the Bernoulli distribution with probability $p_1(u)$, we can rewrite the decision case of Eq. \eqref{eq:sharp-DecDNNF} as follows:
\begin{equation}\label{eq:sharp-DecDNNF-Decision}
\begin{split}
Z(u) &= \frac{Z(ch_0(u))}{2Q(x=0)}Q(x=0) + \frac{Z(ch_1(u))}{2Q(x=1)}Q(x=1)\\
     &= \Expectation{Q}{\frac{Z(ch_x(u))}{2Q(x)}}
\end{split}
\end{equation}
Therefore, if there are two unbiased estimates $\widehat{Z}(ch_0(u))$ and $\widehat{Z}(ch_1(u))$ of $Z(ch_0(u))$ and $Z(ch_1(u))$, an unbiased estimate of $Z(u)$ can be obtained from a set of samples from $Q$. Similarly, if $u$ is a $\AND$-vertex and we have an unbiased estimate $\widehat{Z}(v)$ of $Z(v)$ for each child $v$, then $2^{(1- |Ch(u)|) \cdot |X|} \cdot \prod_{v \in Ch(u)}\widehat{Z}(v)$ is an unbiased estimate of $Z(u)$.

\begin{figure}[ht]
  \centering
  \subfloat[]{\label{fig:DecDNNF:a}
    \centering
    \begin{minipage}[c]{0.44\linewidth}
        \centering
        \includegraphics[width = \textwidth]{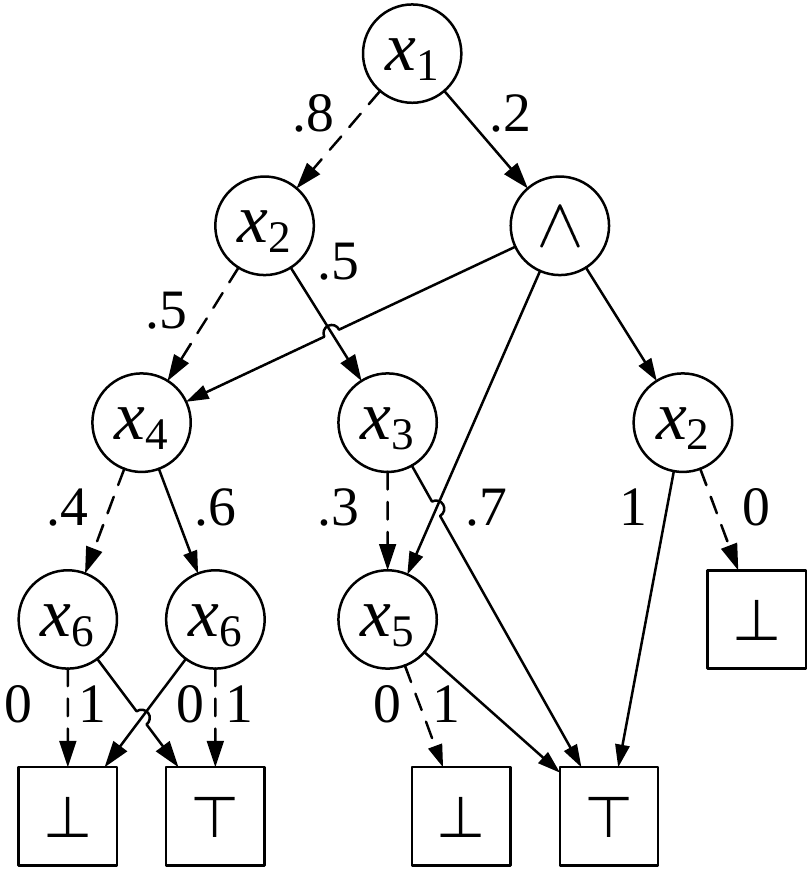}
    \end{minipage}
  }
  \begin{minipage}[c]{0.02\linewidth}
    \hspace{2mm}
  \end{minipage}
  \subfloat[]{\label{fig:DecDNNF:b}
    \centering
    \begin{minipage}[c]{0.48\linewidth}
        \centering
        \includegraphics[width = \textwidth]{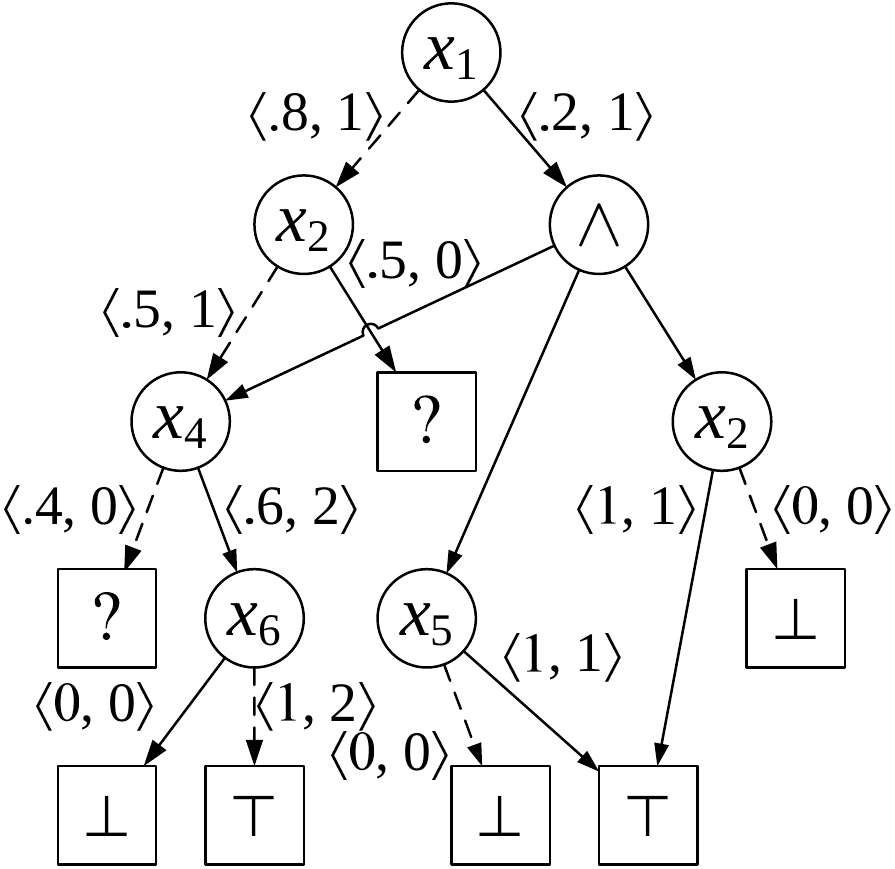}
    \end{minipage}
  }
  \caption{A full Decision-DNNF (a) and a partial Decision-DNNF (b)}\label{fig:DecDNNF}
\end{figure}

\subsection{Partial Decision-DNNF}
\label{sec:DNNF:partial}

A natural idea for overcoming the size explosion of full Decision-DNNF is to prune some sub-graphs in Decision-DNNF. We expect that the remaining vertices still preserve sufficient information to obtain a good estimate of the model number. We replace the pruned sub-graphs in a Decision-DNNF formula with a new type of vertices called \emph{unknown vertices}, and call the resulting DAG a partial Decision-DNNF formula.

\begin{definition}[Partial Decision-DNNF]\label{def:PDecDNNF}
Partial Decision-DNNF is a generalization of full Decision-DNNF by adding a type of new leaf vertices labeled with $?$. Each arc from a decision vertex $v$ is labeled by a pair $\Angle{p_b(v), f_b(v)}$ of estimated marginal probability and visiting frequency, where $b = 0 \text{ or } 1$ means the arc is dashed or solid. For a decision vertices $v$, $p_0(v) + p_1(v) = 1$; $p_c(v) = 0$ iff $sym(ch_b(v)) = \bot$; and $f_b(v) = 0$ iff $sym(ch_b(v)) = ?$.
\end{definition}

Figure \ref{fig:DecDNNF:b} depicts a partial Decision-DNNF.
Hereafter we use $\Angle{?}$ to denote an unknown vertex.
For simplicity, we sometimes use $f(w)$ and $p(w)$ to denote $\Angle{f_0(w), f_1(w)}$ and $\Angle{p_0(w), p_1(w)}$ for each decision vertex $w$.
We can establish a part-whole relationship between partial and full Decision-DNNF formulas:

\begin{definition}\label{def:part}
Let $u$ and $u'$ be partial and full Decision-DNNF vertices, respectively. $\vartheta_u$ is a \emph{part} of $\vartheta_{u'}$ iff $u$ is an unknown vertex, or the following conditions hold:
\begin{compactenum}[(a)]
\item If $u' = \Angle{\bot} \text{ or } \Angle{\top}$, then $u = \Angle{\bot} \text{ or } \Angle{\top}$;
\item If $u' = \Angle{x, ch_0(u), ch_1(u)}$, then $sym(u) = x$, and the partial Decision-DNNF formulas rooted at $ch_0(u)$ and $ch_1(u)$ are parts of the Decision-DNNF formulas rooted at $ch_0(u')$ and $ch_1(u')$, respectively; and
\item If $u' = \Angle{\AND, Ch(u')}$, then $sym(u) = \AND$, $|Ch(u)| = |Ch(u')|$, and each partial Decision-DNNF formula rooted at some child of $u$ is exactly a part of one Decision-DNNF formula rooted at some child of $u'$.
\end{compactenum}
\end{definition}

The partial Decision-DNNF formula depicted in Figure \ref{fig:DecDNNF:b} is a part of the full Decision-DNNF formula depicted in Figure \ref{fig:DecDNNF:a}.
Given a partial Decision-DNNF formula rooted at $u$ that is a part of Decision-DNNF formula rooted at $u'$, the above definition establishes a mapping from the vertices of $\vartheta_u$ to those of $\vartheta_{u'}$. For a vertex $v$ in $\vartheta_u$ corresponding to another vertex $v'$ in $\vartheta_{u'}$ under this mapping, we say that $\Angle{v, v'}$ is a \emph{part-whole pair} of vertices. A full Decision-DNNF formula can be seen as a compact representation for a model set $\Omega$, while a part of this Decision-DNNF formula represents a subset of $\Omega$. In other words, we can use a part of a full Decision-DNNF formula to estimate its model number.
Firstly, a partial Decision-DNNF formula can be used to compute an exact lower (upper) bound of model number:

\begin{proposition}\label{prop:exact-bound}
Let $u$ and $u'$ be, respectively, a partial Decision-DNNF vertex and a full Decision-DNNF vertex over $X$ such that $\vartheta_u$ is a part of $\vartheta_{u'}$. For each unknown vertex $v$ in $\vartheta_u$ corresponding to $v'$ in $\vartheta_{u'}$ under the part-whole mapping, we have a lower (upper) bound $\widetilde{Z}(v)$ of $Z(v')$. A lower (upper) bound of $Z(u')$ can be recursively computed by case analysis in linear time:
\begin{compactenum}[(a)]
\item If $u = \Angle{\bot} \text{ or } \Angle{\top}$, then $\widetilde{Z}(u) = 0 \text{ or } 2^{|X|}$;
\item If $u = \Angle{sym(u), ch_0(u), ch_1(u)}$, then $\widetilde{Z}(u)=\frac{1}{2}\cdot\widetilde{Z}(ch_0(u))+\frac{1}{2}\cdot\widetilde{Z}(ch_1(u))$; and
\item If $u = \Angle{\AND, Ch(u)}$, then $\widetilde{Z}(u) = 2^{(1- |Ch(u)|) \cdot |X|} \prod_{v \in Ch(u)}\widetilde{Z}(v)$.
\end{compactenum}
\end{proposition}

Our main aim is to compute an unbiased estimate of the model number.
For a full Decision-DNNF vertex $u'$, we mentioned that we can compute an unbiased estimate of $Z(u')$ using a set of partial assignments output by \SampleDNNF($u'$).
Actually, we can also obtain an unbiased estimate of $Z(u')$ from some type of partial Decision-DNNFs of $\vartheta_{u'}$ called \emph{randomly} partial Decision-DNNF.
We propose an algorithm called \RandomPart{} in Algorithm \ref{alg:RandomPart} to generate a randomly partial Decision-DNNF based on \SampleDNNF{}.
Firstly, we initialize visiting number and estimate marginal probability for each decision vertex on Lines 1--7.
Secondly, we sample $N$ times and record the visited decision vertices on Lines 8--14.
$(sym(v), b) \in \omega$ on Line 12 means that the arc from $v$ with direction $b$ is visited in calling \SampleDNNF{}.
Obviously, we can update $f_b(v)$ in the Algorithm \SampleDNNF{} instead of using an extra loop to update them on Line 15--17.
Finally, we replace the unvisited vertices with unknown ones.

\begin{algorithm}[htbp]
\caption{\RandomPart($u$, $N$)} \label{alg:RandomPart}
\KwIn{a Decision-DNNF vertex $u$}
\KwOut{a randomly partial Decision-DNNF vertex corresponding to $u$}
\For {each decision vertex $v$ in $\vartheta_u$}{
    $f_0(v) \leftarrow f_1(v) \leftarrow 0$\;
    \lIf {$sym(ch_0(v)) = \bot$} {$p_0(v) \leftarrow 0$}\;
    \lElseIf {$sym(ch_1(v)) = \bot$} {$p_0(v) \leftarrow 1$}\;
    \lElse {let $p_0(v)$ be an estimate of marginal probability of $\vartheta_v$ over $sym(v)$}\;
    $p_1(v) \leftarrow 1 - p_0(v)$\;
}
\For {$i=1$ to $N$}{
    $\omega \leftarrow \SampleDNNF(u)$\;
    Let $V$ be the set of decision vertices visited in the current calling of $\SampleDNNF(u)$\;
    \For {each vertex $v \in V$ and each Boolean value $b$}{
        \lIf {$(sym(v), b) \in \omega$} {$f_b(v) \leftarrow f_b(v) + 1$}\;
    }
}
\For {each decision vertex $v$ and each Boolean value $b$}{
    \lIf {$f_b(v) = 0$} {substitute $\Angle{?}$ for $ch_b(v)$ in $\vartheta_u$}\;
}
\KwRet $u$
\end{algorithm}

We can use a randomly partial Decision-DNNF to compute an unbiased estimate of the model number of the corresponding full Decision-DNNF. Note that we assume that 0 divided by 0 equals 0.

\begin{proposition}\label{prop:unbiased}
Let $\Angle{u, u'}$ be a part-whole pair of Decision-DNNF vertices over $X$, where $u$ is randomly partial Decision-DNNF generated by calling \PartialKC($u'$, $N$). We can recursively compute an unbiased estimate of $Z(u')$ in linear time by case analysis (for each unknown vertex $v$ in $\vartheta_u$, we assume that $\widehat{Z}(v)$ equals a random value $z$, which does not impact the value of $\widehat{Z}(u)$):
\begin{compactenum}[(a)]
\item if $u = \Angle{\bot} \text{ or } \Angle{\top}$, then $\widehat{Z}(u) = 0 \text{ or } 2^{|X|}$;
\item if $u = \Angle{sym(u), ch_0(u), ch_1(u)}$, then
\begin{align*}
\widehat{Z}(u)&=\frac{\widehat{Z}(ch_0(u)) \cdot f_0(u)}{2p_0(u) \cdot (f_0(u) + f_1(u))}+\\
&\mathrel{\phantom{=}}\frac{\widehat{Z}(ch_1(u))\cdot f_1(u)}{2p_1(u) \cdot (f_0(u) + f_1(u))}
\end{align*}; and
\item if $u = \Angle{\AND, Ch(u)}$, then $\widehat{Z}(u) = 2^{(|Ch(u)| - 1) \cdot |X|} \cdot \prod_{v \in Ch(u)}\widehat{Z}(v)$.
\end{compactenum}
\end{proposition}

\begin{example}
Let $u$ be the root of the Decision-DNNF in Figure \ref{fig:DecDNNF:a}. The exact number of models of $\vartheta_u$ over $\{x_1, \ldots, x_6\}$ is 24. We assume that two callings of $\SampleDNNF(u)$ output two partial assignments $\{x_1 = 0, x_2 = 0, x_4 = 1, x_6 = 0\}$ and $\{x_1 = 1, x_2 = 1, x_4 = 1, x_4 = 1, x_6 = 0\}$. Then \RandomPart($u$, $2$) will output the partial Decision-DNNF in Figure \ref{fig:DecDNNF:b}. According to Proposition \ref{prop:unbiased}, we compute $\widehat{Z}(u)$ in a bottom-up way. It is obvious that for a vertex $v$ representing a literal, $\widehat{Z}(v) = 2^5$. This yields the following results (we denote $ch_0(ch_0(u))$ by $u_{00}$):

\begin{equation*}
\widehat{Z}(u_{00}) = \frac{z \times 0}{2 \times 0.4 \times 2} + \frac{2^5  \times 2}{2 \times 0.6 \times 2} = \frac{2^4}{0.6};
\end{equation*}
\begin{equation*}
\widehat{Z}(ch_0(u)) = \frac{\widehat{Z}(u_{00}) \times 1}{2 \times 0.5 \times 1} + \frac{z \times 0}{2 \times 0.5 \times 1} = \frac{2^4 }{0.6};
\end{equation*}
\begin{equation*}
\widehat{Z}(ch_1(u)) = 2^{ - 2 \times 6}  \times \widehat{Z}(u_{00}) \times 2^5  \times 2^5  = \frac{2^2}{0.6};
\end{equation*}
\begin{equation*}
\widehat{Z}(u) = \frac{\widehat{Z}(ch_0(u))  \times 1}{2 \times 0.8 \times 2} + \frac{\widehat{Z}(ch_1(u))  \times 1}{2 \times 0.2 \times 2} = \frac{10}{0.6}.
\end{equation*}

\end{example}

\section{Partial Knowledge Compilation}
\label{sec:partialkc}

We aim at estimating the model numbers for hard CNF formulas, which cannot be compiled within the tolerated time and limited space. Obviously, we cannot accomplish this goal by first compiling a CNF formula into an equivalent Decision-DNNF formula, and then generating a randomly partial Decision-DNNF formula to estimate the number of models (in fact, we can obtain the exact number of models if a full compilation is possible).
We directly generate a randomly partial Decision-DNNF formula from the CNF formula rather than from an equivalent full Decision-DNNF formula. The process of generating a partial Decision-DNNF formula from a CNF formula is called \emph{partial KC}. We propose an algorithm called \PartialKC{} (in Algorithm \ref{alg:PartialKC}) to generate a randomly partial Decision-DNNF formula.

\begin{algorithm}[htbp]
\caption{\PartialKC($\varphi$, $N$)} \label{alg:PartialKC}
\KwIn{a satisfiable CNF formula $\varphi$, a positive integer $N$, and an implicit hash table $H$ storing the results of calling \MicroKC{}}
\KwOut{a randomly partial Decision-DNNF formula corresponding to some Decision-DNNF formula equivalent to $\varphi$}
\lFor {$i=1$ to $N$}{$\MicroKC(\varphi)$}\;
\KwRet $H(\varphi)$
\end{algorithm}

The algorithm \PartialKC{} consists of $N$ callings of \MicroKC{} presented in Algorithm \ref{alg:MicroKC}.
We use the hash table $H$ to store the current compiled result implicitly.
Each calling of \MicroKC{} will update the hash table and thus implicitly enlarge the current compiled result rooted at $H(\varphi)$.
On Lines 2--3, we deal with the cases where $\varphi$ is 1 or a literal.
On Lines 5--15, we deal with the case of the initial calling of \MicroKC{} on $\varphi$.
We decompose $\varphi$ into a set of sub-formulas without sharing variables on Line 5.
We require that all implied literals are extracted.
Therefore, for each sub-formula $\psi_i$ with more than one variable, and each variable $x$, both $\psi_i|_{x = 0}$ and $\psi_i|_{x = 1}$ are satisfiable.
On Lines 7--13, we deal with the case where $\varphi$ is not decomposable.
We create a decision vertex $u$ labeled with a variable $x$ from $Vars(\varphi)$.
We estimate the marginal probability of $\varphi$ over $x$ and sample a Boolean value $b$ with this probability.
Note that the variance of our model counting method strongly depends on the accuracy of the estimate.
We generate children of $u$ and update the information about probability and frequency on Lines 11--13.
On Line 14, we deal with the case where $\varphi$ is decomposable, and recursively call \MicroKC{} for each sub-formula.
On Lines 16--20, we deal with the case on the repeated calling of \MicroKC{} on $\varphi$.

\begin{algorithm}[!htb]
\caption{\MicroKC($\varphi$)} \label{alg:MicroKC}
\KwIn{a satisfiable CNF formula $\varphi$, and an implicit hash table $H$ storing the results of previous callings of \MicroKC{} }
\KwOut{implicitly output a new randomly partial Decision-DNNF formula of $\varphi$ via hash table $H$}
$v \leftarrow H(\varphi)$\;
\lIf {$\varphi$ is true} {$H(\varphi) \leftarrow \Angle{\top}$}\;
\lElseIf {$\varphi$ is a literal $l$} {$H(\varphi) \leftarrow \Angle{l}$}\;
\uElseIf {$v = nil$} {
    Decompose $\varphi$ into $\{\psi_1, \ldots, \psi_n\}$ that includes all implied literals\;
    \uIf {$n = 1$} {
        Choose a variable $x$ from $\varphi$\;
        Let $p$ be an estimate of marginal probability of $\varphi$ over $x$\;
        Sample a Boolean value $b$ with probability $p$\;
        Create a decision vertex $u$ with $sym(u) = x$ \;
        $ch_b(u) \leftarrow \MicroKC(\varphi|_{x = b})$;  $ch_{1-b}(u) \leftarrow \Angle{?}$\;
        $p_{0}(u) \leftarrow 1-p$; $p_{1}(u) \leftarrow p$\;
        $f_{b}(u) \leftarrow 1$; $f_{1-b}(u) \leftarrow 0$\;
    }
    \lElse {$u \leftarrow \Angle{\AND, \{\MicroKC(\psi_i): 1 \le i \le n\}}$}\;
    Replace $v$ with $u$ in the hash table and the resulting partial Decision-DNNF formula\;
}
\uElseIf {$v$ is a decision vertex} {
        Sample a Boolean value $b$ with probability $p_1(v)$\;
        $f_b(v) \leftarrow f_b(v) + 1$\;
        $ch_b(v) \leftarrow \MicroKC(\varphi|_{sym(v) = b})$\;
}
\lElse {$Ch(v) \leftarrow \{\MicroKC(H^{-1}(w)): w \in Ch(v)\}$}\;
\end{algorithm}

It is easy to see that if we have sufficient time and memory, \PartialKC($\varphi$, $\infty$) will output a partial Decision-DNNF formula without any unknown vertex; that is, the result can be seen as a full Decision-DNNF formula that is equivalent to $\varphi$.
If the visiting frequency is not considered, this algorithm converges to a final equivalent Decision-DNNF formula for each CNF formula as $N$ approaches infinity.
Assume that this final Decision-DNNF is rooted at $u'$.
Obviously, the outputs of \PartialKC($\varphi$, $N$) and \RandomPart($u'$, $N$) are independent and identically distributed.
Therefore, we can draw the following conclusion:
\begin{proposition}\label{prop:PartialKC}
Given a CNF formula $\varphi$, there exists some equivalent full Decision-DNNF formula rooted at $u'$ such that \PartialKC($\varphi$, $N$) outputs a partial Decision-DNNF formula which is part of $\vartheta_{u'}$.
\end{proposition}

Since each calling of \MicroKC($\varphi$) will introduce at most $3 \cdot |Vars(\varphi)|$ edges into the output of \PartialKC($\varphi$, $N$), the output of \PartialKC($\varphi$, $N$) has at most $3 \cdot |Vars(\varphi)| \cdot N$ edges. According to Propositions \ref{prop:unbiased}--\ref{prop:PartialKC}, the following conclusion can be drawn:

\begin{theorem}\label{thm:PartialKC}
Given a CNF formula $\varphi$ and an integer $N$, we can compute an unbiased estimate of the model number of $\varphi$ from the output of \PartialKC($\varphi$, $N$) in $O(|Vars(\varphi)| \cdot N)$.
\end{theorem}
\begin{example}\label{exam:PartialKC}
We run \PartialKC{} on the formula $\varphi = (x_2 \OR x_4 \OR x_6 ) \AND (x_2 \OR \NOT x_4 \OR \NOT x_6 ) \AND (x_1 \OR \NOT x_2  \OR x_3 \OR x_5 ) \AND (\NOT x_1 \OR x_4 \OR x_6 ) \AND (\NOT x_1 \OR \NOT x_4  \OR \NOT x_6 ) \AND (\NOT x_1 \OR x_2 \OR \NOT x_4 ) \AND (\NOT x_1 \OR x_2 \OR x_4 ) \AND (\NOT x_1 \OR \NOT x_2 \OR x_5 )$. Assume that the variable with the minimum subscript is chosen on Line 7.
For the first calling of \MicroKC($\varphi$), the condition on Line 9 is satisfied. We choose $x_1$ on Line 7, and estimate that the marginal probability of $\varphi$ over $x_1$ is about 0.2. We assume that 0 is sampled on Line 9. Then \MicroKC($\varphi_0$) is recursively called, where $\varphi_0 = (x_2 \OR x_4 \OR x_6 ) \AND (x_2 \OR \NOT x_4 \OR \NOT x_6 ) \AND (\NOT x_2 \OR x_3 \OR x_5 )$. Similarly, the condition on Line 6 is satisfied. We choose $x_2$ on Line 7, and estimate that the marginal probability of $\varphi_0$ over $x_2$ is about 0.5. We assume that false is sampled on Line 9. Then \MicroKC($\varphi_{00}$) is recursively called, where $\varphi_{00} = (x_4  \OR x_6 ) \AND (\NOT x_4  \OR \NOT x_6 )$. Similarly, the condition on Line 6 is satisfied. We choose $x_4$ on Line 7, and estimate that the marginal probability of $\varphi_{00}$ over $x_4$ is about 0.6. We assume that 1 is sampled on Line 9. Then \MicroKC($\NOT x_6$) is recursively called, and then $\Angle{\NOT x_6}$ is returned.
\MicroKC($\varphi_{00}$) returns $v_{00} = \Angle{x_4, \Angle{?}, \Angle{\NOT x_6}}$ with $f(v_{00}) = \Angle{0, 1}$ and $p(v_{00}) = \Angle{0.4, 0.6}$.
\MicroKC($\varphi_{0}$) returns $v_{0} = \Angle{x_2, v_{00}, \Angle{?}}$ with $f(v_{0}) = \Angle{1, 0}$ and $p(v_{00}) = \Angle{0.5, 0.5}$.
\MicroKC($\varphi$) returns $v = \Angle{x_1, v_0, \Angle{?}}$ with $f(v) = \Angle{1, 0}$ and $p(v) = \Angle{0.8, 0.2}$.
For the second calling of \MicroKC($\varphi$), the condition on Line 16 is satisfied. We assume that 1 is sampled on Line 17. Then \MicroKC($\varphi_1$) is recursively called, where $\varphi_1 = (x_2 \OR x_4  \OR x_6 ) \AND (x_2 \OR \NOT x_4  \OR \NOT x_6 ) \AND (x_4  \OR x_6 ) \AND (\NOT x_4  \OR \NOT x_6 ) \AND (x_2 \OR \NOT x_4 ) \AND (x_2 \OR x_4 ) \AND (\NOT x_2 \OR x_5 )$. We decompose $\varphi_1$ into $\{(x_4  \OR x_6 ) \AND (\NOT x_4  \OR \NOT x_6 ), x_2, x_5\}$. Then \MicroKC($\varphi_{00}$), \MicroKC($x_2$), and \MicroKC($x_5$) are recursively called. The last two callings return $\Angle{x_2}$ and $\Angle{x_5}$. For the calling of \MicroKC($\varphi_{00}$), the condition on Line 16 is satisfied. We assume that 1 is sampled on Line 17. Then \MicroKC($\NOT x_6$) is recursively called, and then $\Angle{\NOT x_6}$ is returned.
\MicroKC($\varphi_{00}$) returns $v_{00} = \Angle{x_4, \Angle{?}, \Angle{\NOT x_6}}$ with $f(v_{00}) = \Angle{0, 2}$ and $p(v_{00}) = \Angle{0.4, 0.6}$.
\MicroKC($\varphi_{1}$) returns $v_{1} = \Angle{\AND, \{v_{00}, \Angle{x_2}, \Angle{x_5}\}}$.
\MicroKC($\varphi$) returns $v = \Angle{x_1, v_0, v_1}$ with $f(v) = \Angle{1, 1}$ and $p(v) = \Angle{0.8, 0.2}$.
Finally, we generate the partial Decision-DNNF in Figure \ref{fig:DecDNNF:b}.
\end{example}

\subsection{KC Technologies for Reducing the Variance}

\MicroKC{} can be seen as a sampling procedure equipped with KC technologies, and the variance of the randomly partial Decision-DNNF depends on three main factors.
Firstly, the variance of a single calling of \MicroKC{} depends on the number of sampling Boolean values on Lines 9 and 17.
The less the samples from the Bernoulli distributions, the smaller the variance.
Secondly, the variance also depends on the number of \MicroKC{} callings when fixing their total time.
We expect to accelerate the running of \MicroKC{}.
Finally, the variance depends on the quality of proposal distribution.
In the following paragraphs, three KC technologies, which can be used for reducing the variance, are explained.

The first technology is dynamic decomposition on Line 5. We employ a SAT solver to compute the implied literals of a formula, and use these implied literals to simplify the formula. Then we decompose the residual formula according to the corresponding primal graph. We can reduce the variance from the following three aspects:
\begin{compactitem}[$\bullet$]
\item After extracting all implied literals, on the one hand, we can reduce the number of sampling Boolean values from the Bernoulli distributions. On the other hand, for each $\psi_i$ on Line 5 with more than one variable and each variable $x \in Vars(\psi_i)$, we have that $\psi_i \not\models \NOT x$ and $\psi_i \not\models x$. That is, the sampling on Lines 9 and 17 is backtracking-free, which remedies the rejection problem of sampling.
\item Similar to the approach used in \cite{Gogate:Dechter:11}, we reduce the variance by sampling from a subset of the variables, a technique also known as Rao-Blackwellization. In our implementation, Line 2 in Algorithm \ref{alg:MicroKC} is replaced by the following statement: if $\varphi$ is trivial, then we assign $H(\varphi)$ as a known vertex labeled with $Z(\varphi)$. We can detect the trivialness of a CNF formula based on its variable number or some other parameters. After decomposing, we can detect more than one trivial formula, and thus greatly reduce the variance.
\item More virtual samples can be provided to reduce the variance. For example, sampling twice for each sub-formula on Line 14 yields $2^n$ virtual samples.
\end{compactitem}

The second technology is the component caching implemented in hash table $H$.
In different callings of \MicroKC{}, the same sub-formula may need to be processed several times.
With the use of component caching, we can save the time of computing implied literals, and thus accelerate the sampling.
On the other hand, we can reduce the variance by merging the callings of \MicroKC{} on the same formula.
Consider Example \ref{exam:PartialKC} again. We call \MicroKC{} twice on $\varphi_{00}$.
The corresponding variance is obviously smaller than that of a single calling of \MicroKC{}.
Then, the variance of a randomly partial Decision-DNNF formula is reduced.
In our implementation, the advanced component caching in sharpSAT \cite{sharpSAT} was adopted.
If the use of hash table $H$ in \MicroKC{} is canceled, then each calling of \MicroKC{} will return a sample.
These samples can be used to estimate the model number, and the resulting counter will be referred to as \PartialKC{}-single.
In the following section, it will be shown that hash table $H$ can accelerate the running of \PartialKC{}, in contrast to \PartialKC{}-single.

The third technology is that of generating the proposal distribution based on another form of partial KC.
Without consideration of the dynamic decomposition, each calling of \MicroKC{} can be seen as a process of importance sampling, where the resulting partial Decision-DNNF is treated as the proposal distribution.
Similar to importance sampling, it is easy to see that the variance of using \PartialKC{} to estimating model number depends on the quality of estimating the marginal probability on Line 8.
If the estimated marginal probability is equal to the true one, \PartialKC{} will yield an optimal (zero variance) estimate.
It is obvious that the exact marginal probability can be calculated via an equivalent Decision-DNNF formula.
However, we cannot afford the computational cost for compiling a CNF formula into full Decision-DNNF.
Alternatively, we estimate the marginal probability by compiling the formula into a partial Decision-DNNF.
The compiling algorithm is presented in Algorithm \ref{alg:PartialOBDDL}, which is based on the classical DPLL algorithm.
To avoid a high computational cost, we search for several variables in this DPLL-based algorithm.
A small size of partial assignments often impedes the further decomposition of the CNF formula, except for the implied literals.
Therefore, we choose to compile the CNF formula into a partial \OBDDL{}, which is a special class of partial Decision-DNNF formula.
Firstly, we generate a partial \OBDDL{}.
Secondly, we substitute $\Angle{\top}$ for each vertex $\Angle{?}$, which turns the resulting DAG into a full \OBDDL{}.
The marginal probability of the latter is used to estimate that of a CNF formula over $x$.
Consider the formula $\varphi$ in Example \ref{exam:PartialKC} again. Let $X = \{x_1, x_2\}$ with $x_1 \prec x_2$. Then \PartialOBDDL($\varphi$, $X$, $\prec$) will output $\Angle{x_1, \Angle{x_2, \Angle{?}, \Angle{?}}, \Angle{\AND, \{\Angle{x_2}, \Angle{?}\}}}$. Therefore, the marginal probability of $\varphi$ over $x_1$ can be roughly estimated as 0.33.
\begin{algorithm}[!htp]
\caption{\PartialOBDDL($\varphi$, $X$, $\prec$)} \label{alg:PartialOBDDL}
\KwIn{a CNF formula $\varphi$, a subset $X$ of $Vars(\varphi)$, and a total order $\prec$}
\KwOut{a partial \OBDDL{} over $\prec$ }
Let $L$ be a set of implied literals over $X$ computed by implicit BCP \;
Simplify $\varphi$ using the implied literals in $L$ \;
\lIf {$\varphi$ has some empty clause} {\KwRet $\Angle{\bot}$} \;
\lElseIf {$\varphi = \emptyset$} {$u \leftarrow \Angle{\top}$} \;
\lElseIf {$X = \emptyset$} {$u \leftarrow \Angle{?}$} \;
\Else {
    $x \leftarrow \min_{\prec}X\cap Vars(\varphi)$ \;
	$v \leftarrow \PartialOBDDL( \varphi |_{x=0}, X \setminus \{x\}, \prec )$ \;
	$w \leftarrow \PartialOBDDL( \varphi |_{x=1}, X \setminus \{x\}, \prec )$ \;
    $u \leftarrow \Angle{x, v, w}$ \;
}
Let $V$ be the set of vertices representing $L$ \;
\lIf {$V = \emptyset$} {\KwRet $u$} \;
\lElse {\KwRet $\Angle{\AND, V \cup \{u\}}$} \;
\end{algorithm}

\section{Preliminary Experimental Results}
\label{sec:experiment}

We mainly focus on counting models for the instances that are hard to be solved for exact counters. To allow comparison, approximate model counters are evaluated by comparing the quality of the generated lower-bounds (the higher the better). We use the scheme described in \cite{Gomes:etal:07} to compute lower bounds. According to Markov's inequality, for an unbiased or under estimate $\widehat{Z}$ of count $Z$, $Pr[\widehat{Z} > cZ] < 1/c$. Given $m$ estimates of $Z$, let $\widehat{Z}^*$ be the minimum one. For the event that $\widehat{Z}^*/c$ does not exceed $Z$, the respective probability does not exceed $1-c^{-m}$. Therefore, if $\delta \le 1-c^{-m}$, then $\widehat{Z}^*$ will provide a lower bound of $Z$ with probability at least $\delta$. In our experiments, we set the confidence $\delta = 0.99$ and thus set $m=7$ and $c = 1.9307$. Therefore, the minimum estimate of seven ones is divided by 1.9307 to denote the lower bound of the model number. Specifically, the lower bounds obtained by counters \PartialKC{}, SampleSearch/LB \cite{Gogate:Dechter:11}, and SearchTreeSampler \cite{Ermon:etal:13} were compared.

Our benchmark suite consisted of problems arising from four domains. The first one involved the Field Programmable Gate Array (FPGA) routing instances, which were constructed by reducing FPGA detailed routing problems into CNF formulas. The remaining three domains of benchmark instances corresponded to the Langford¡¯s problem domain, the normalized Latin squares domain, and the ISCAS89 combinational circuits, respectively. The first, second, and third domains are mostly involved in the testing of approximate counters, while the fourth one is frequently used to test full compilers.
\begin{table*}[!htbp]
\centering
\small
\aboverulesep = 0.2ex
\belowrulesep = 0.2ex
\caption{Comparative lower bounds on model numbers and numbers of samples output by \PartialKC{}, SampleSearch/LB, and SearchTreeSampler} \label{tab:expri:overall}
\renewcommand{\arraystretch}{1.35}
\begin{tabularx}{\linewidth}{c@{\hspace{1pt}}c@{\hspace{1pt}}*{2}{@{\hspace{2pt}}>{\centering\arraybackslash}c@{\hspace{2pt}}}*{4}{@{\hspace{2pt}}>{\centering\arraybackslash}X@{\hspace{2pt}}}}\toprule
\multirow{2}*{Problem} & \multirow{2}*{$\Angle{n, c, w}$} & \multicolumn{2}{c}{\PartialKC{}} & \multicolumn{2}{c}{SampleSearch/LB} & \multicolumn{2}{c}{SearchTreeSampler} \\\cmidrule{3-8}
& & $Z$ & $N$ (+) & $Z$ & $N$ & $Z$ & $N$ \\\midrule
9symml\_gr\_rcs\_w6     & $\Angle{1554, 29119, 575}$ & \textbf{5.98E+84} & 73803 (152\%) & 4.48E+82 & 205285 & 4.86E+83 & 19417  \\
apex7\_gr\_2pin\_w5     & $\Angle{1983, 15358, 168}$ & \textbf{1.60E+95} & 19579 (20\%) & 6.31E+92 & 82881 & 3.82E+91 & 9526 \\
c880\_gr\_rcs\_w7       & $\Angle{4592, 61745, 925}$ & \textbf{1.18E+267} & 4312 (22\%) & 5.1316E+252 & 6083 & 6.26E+266 & 1729 \\
vda\_gr\_rcs\_w9        & $\Angle{6498, 130997, 2132}$ & \textbf{5.62E+310} & 165 (0\%) & n/a & 0 & inf & n/a \\\midrule[0.01em]
lang16      & $\Angle{1024, 32320, 408}$ & \textbf{3.35E+08} & 71645 (24\%) & 1.95E+08 & 180675 & n/a & 0 \\
lang19      & $\Angle{1444, 54226, 602}$ & \textbf{2.61E+11} & 54371 (41\%) & 5.35E+10 & 40081 & n/a & 0 \\\midrule[0.01em]
ls12-norm   & $\Angle{1221, 11231, 821}$ & 4.04E+42 & 140 (9\%) & 2.05E+42 & 27885 & \textbf{6.97E+43} & 150 \\
ls13-norm   & $\Angle{1596, 16248, 1092}$ & 1.80E+53 & 46 (7\%) & 2.52E+52 & 1752 & \textbf{4.88E+54} & 7 \\\midrule[0.01em]
s13207      & $\Angle{8651, 19116, 76}$ & \textbf{2.72E+210} & 0.2G (\textgreater{}1000\%) & 1.35E+209 & 267925 & 1.1682E+207 & 1145 \\
s15850      & $\Angle{10383, 23417, 96}$ & \textbf{2.79E+183} & 18966 (247\%) & 1.07E+180 & 164460 & 6.36E+181 & 818 \\\bottomrule
\end{tabularx}
\end{table*}

We conducted experiments on a computer with a 64-bit eight-core 3.4 GHz CPU and 16GB RAM.
Each run of the model counter for each instance was allowed five hours.
The results obtained for ten representative instances are listed in Table \ref{tab:expri:overall}.
It is noteworthy that none of these instances, except for s13207, can be solved in five hours by the state-of-the-art exact counter sharpSAT.
The second column contains various statistical information about the instances: $n$ is the number of variables, $c$ is the number of clauses and $w$ is the treewidth computed using the min-fill heuristic.
In the remaining columns, $Z$ indicates the lower boound, $N$ indicates the total calling number of \MicroKC{}, the total number of samples, and the total number of sets of solutions for \PartialKC{}, SampleSearch/LB and SearchTreeSampler in the corresponding seven runs, respectively.
It is easy to see that each calling of \MicroKC{} can generate a sample.
Figures in brackets corresponding to $N$ (+) column indicate the increase in the sampling rate of \PartialKC{}, as compared to \PartialKC{}-single, in percent.
For the instance on which SearchTreeSampler reported \textquoted{inf}, we guess that some variable of floating point number in the program overflowed.

Our experimental results show that \PartialKC{} succeeded in reporting lower bound for each instance, while SampleSearch/LB and SearchTreeSampler failed to report any results in one and two cases, respectively, out of the ten instances.
As compared to SampleSearch/LB, \PartialKC{} yielded better lower bounds for all instances.\footnote{Actually, SampleSearch/LB yielded worse lower bounds for all instances than \PartialKC{}-single, which yielded worse lower bounds for most instances than \PartialKC{}.}
For five instances, SampleSearch/LB underestimated the counts by at least one order of magnitude, in contrast to \PartialKC{}.
We can also observe that \PartialKC{} computed higher lower bounds than SearchTreeSampler on eight out of the ten instances, and the former scales better than the latter.
The performance of \PartialKC{} in the normalized Latin squares domain is worse than that of SearchTreeSampler, because the high costs of computing implied literals for the instances under partial assignments make the sampling very slow.
The experimental results also show that \PartialKC{} can generate more samples than \PartialKC{}-single for all instances except vda\_gr\_rcs\_w9, due to the KC technology component caching.
For three instances, samples generated by \PartialKC{} are at least twice as many as those generated by \PartialKC{}-single.
As applied to the same problem, it is easy to understand that the sampling rate of \PartialKC{} increases with the running time.
Therefore, \PartialKC{} is more effective for relatively easy problems which are still hard for exact model counting.

\section{Conclusions}
\label{sec:conclude}

A new representation partial Decision-DNNF, and a new approach \PartialKC{} for approximate model counting were proposed, to mitigate the explosion in sizes, which is intrinsic to all full KC methods.
For each CNF formula, \PartialKC{} can generate a randomly partial Decision-DNNF formula, which can be used to compute an unbiased estimate of the model number of the corresponding formula.
A single calling of \PartialKC{} consists of multiple callings of \MicroKC{}, and each of the latter callings can be seen as a process of importance sampling equipped with KC technologies.
The proposal distribution chosen in \MicroKC{} is based on the partial Decision-DNNF formula itself.
Each decision vertex in partial Decision-DNNF is labeled with a probability which is estimated via a partial \OBDDL{} generating algorithm.
The experimental results show that \PartialKC{} are more accurate than both SampleSearch and SearchTreeSampler, and scales better than SearchTreeSampler.
This study build new connections between the two fields of model counting and KC, which makes that the progress of approximate model counting can directly benefit from the development of KC technologies.

\subsubsection*{Acknowledgements}

We are grateful to the anonymous reviewers who suggested improvements. We would also thank Stefano Ermon for providing useful information about their software SearchTreeSampler. This research is supported by the National Natural Science Foundation of China under grants 61402195 and 61472161, and by the China Postdoctoral Science Foundation under grant 2014M561292.

\newpage
\bibliographystyle{elsarticle-num}
\bibliography{mypublications,papers,books,softwares}







\end{document}